\documentclass[a4paper]{article}

\usepackage[english]{babel}
\usepackage[utf8x]{inputenc}
\usepackage[T1]{fontenc}

\usepackage[a4paper,top=3cm,bottom=3.5cm,left=3cm,right=3cm,marginparwidth=1.75cm]{geometry}
\usepackage{amsmath}
\usepackage{amssymb}
\usepackage{amsfonts}
\usepackage{bm}
\usepackage{subfig}
\usepackage{float}
\usepackage{graphicx}
\usepackage{microtype}
\usepackage{amsthm}
\usepackage{enumerate}
\usepackage[breaklinks=true, colorlinks=true, allcolors=blue]{hyperref}
\usepackage{color}
\usepackage[colorinlistoftodos]{todonotes}
\usepackage{algorithm}
\usepackage[noend]{algpseudocode}
\usepackage{natbib}

\newtheorem{thm}{Theorem}
\newtheorem{lem}[thm]{Lemma}

\newtheorem{cor}[thm]{Corollary}
\newtheorem{defn}[thm]{Definition}

\title{Generalization of an Upper Bound on the Number of
Nodes Needed to Achieve Linear Separability}
\author{Marjolein Troost, Katja Seeliger and Marcel van Gerven}
\date{}

\begin{document}
\maketitle
\begin{center}
  Radboud University\\
  Donders Institute for Brain, Cognition and Behaviour\\
  Nijmegen, The Netherlands\\
    \end{center}
~\\
~\\
\textbf{Abstract:} An important issue in neural network research is how to choose the number of nodes and layers such as to solve a classification problem. We provide new intuitions based on earlier results by \cite{an2015} by deriving an upper bound on the number of nodes in networks with two hidden layers such that linear separability can be achieved. Concretely, we show that if the data can be described in terms of $N$ finite sets and the used activation function $f$ is non-constant, increasing and has a left asymptote, we can derive how many nodes are needed to linearly separate these sets. This will be an upper bound that depends on the structure of the data. This structure can be analyzed using an algorithm. For the leaky rectified linear activation function, we prove separately that under some conditions on the slope, the same number of layers and nodes as for the aforementioned activation functions is sufficient. We empirically validate our claims.\\
~\\
\textbf{Keywords:} Artificial neural networks, linear separability, disjoint convex hull decomposition.


\section{Introduction}\label{sec:intro}

Artificial neural networks perform very well on classification problems. They are known to be able to linearly separate almost all input sets efficiently. However, it is not generally known how the artificial neural networks actually obtain this separation so efficiently. Therefore, it is difficult to choose a suitable network to separate a particular dataset. Hence, it would be useful if, given a dataset used for training and a chosen activation function, one can analytically derive how many layers and nodes are necessary and sufficient for achieving linear separability on the training set. Some steps in this direction have already been taken.\\
~\\
In \cite{an2015} it has been shown for rectified linear activation functions that the number of hidden layers sufficient for linearly separating any number of (finite) datasets is two (follows from universality as well) and that the number of nodes per layer can be determined using disjoint convex hull decompositions.  \cite{yuan2003} have provided estimates for the number of nodes per layer in a two-layer network based on information-entropy. \cite{fujita1998} has done the same based on statistics by adding extra nodes one by one. Another approach by \cite{kurkova1997} is to calculate how well a function can be approximated using a fixed number of nodes. Recently, a paper (\cite{ShwartzZiv2017}) has appeared that uses the information plane and information bottleneck to understand the inner workings of neural networks. \cite{baum1988} has shown that a single-layer network can approximate a random dichotomy with only $N/d$ units for an arbitrary set of $N$ points in general position in $d$ dimensions. He also makes the link to the Vapnik-Chervonenkis dimension of the network. In this work we do not use statistics to achieve an estimate of the number of nodes but rather simple algebra to obtain an absolute upper bound, in the spirit of \cite{an2015} and  \cite{baum1988}. In contrast to \cite{an2015}, we will obtain this bound for multiple activation functions and in contrast to \cite{baum1988} the bound will hold for arbitrary finite sets.\\
~\\
\begin{sloppypar}
It is well-known that neural networks with one hidden layer 
are universal approximators (e.g. \cite{hornik1989,arteaga2013} or more recently \cite{sonoda2017}). However, even though we know there should exist a network that can linearly separate two arbitrary finite sets, we do not know which one it is. Choosing the wrong kind of network can lead to severe overfitting and reduced performance on the test set~\cite{yuan2003}. Therefore, it is useful to have an upper bound on the number of nodes. The upper bound can aid in choosing an appropriate network for a task. With this in mind, we aim to give a theoretical upper bound on the size of a network with two hidden layers in terms of nodes, that is easily computable for any finite input sets that need to be separated.\\
\end{sloppypar}

~\\
The rest of this work is organized as follows:  In Section~\ref{sec:def} we repeat some of the definitions from~\cite{an2015} and we give a direct extension of two of their theorems for which their proof does not need to be changed. In Section~\ref{sec:main} we present our main theorem, which generalizes the two theorems from Section~\ref{sec:def} to a larger class of activation functions. In Section~\ref{sec:corr} we add some corollaries and refer to an extension to multiple sets that is given in~\cite{an2015}, we also provide an algorithm to estimate the upper bound on the number of nodes. We show simulation results that support our claims in Section~\ref{sec:experiments} and conclude with some final remarks in Section~\ref{sec:disc}.\\

\section{Achieving linear separability}\label{sec:def}

We want to emphasize that the following definitions and theorems (Definition~\ref{def:disjointchdecomp}, Theorems~\ref{thm:linsep} and \ref{thm:convsep} and Corollary~\ref{cor:multiple}) are due to  \cite{an2015} and are repeated here for convenience. We took the liberty of adapting some of these definitions for clarity and giving slightly stronger versions of their Theorems 4 and 5 in Theorems~\ref{thm:linsep} and \ref{thm:convsep} which follow directly from the proof given by \cite{an2015}.\\
~\\
Throughout the article, we will use the following notation and conventions: all sets of data points are finite. We use $f$ to denote a non-constant activation function that is always applied element-wise to its argument. So \begin{equation*}f((x_1,x_2,\ldots,x_n)^T)=(f(x_1),f(x_2)\ldots,f(x_n))^T.\end{equation*} We define the convex hull of a set as the set of all convex combinations of the points in the set. In set notation: \begin{equation*} C(X)=\left\lbrace\sum_{i=1}^{|X|} \alpha_i x_i\,\biggr\rvert\, \forall i ~~\alpha_i\geq 0,~~ \sum_{i=1}^{|X|} \alpha_i = 1\right\rbrace.
\end{equation*} We will now first define what is meant by a disjoint convex hull decomposition. $\mathbb{R}$ is the set of real numbers. 
\begin{defn}\label{def:disjointchdecomp}
	Let $X_k$, $k\in \{1,\ldots,m\}$, be $m$ disjoint, finite sets in $\mathbb{R}^n$. A decomposition of $X_1, \ldots, X_m$, $X_k = \bigcup_{i=1}^{L_k} X_k^i,$ with $L_k\geq 1$ is called a disjoint convex hull decomposition if the unions of the convex hulls of $X_k^i$, \begin{equation*}
	\hat{X}_k \triangleq \bigcup_{i=1}^{L_k} C(X_k^i)\,,
\end{equation*}
 are still disjoint. I.e. for all $k\neq l$: $\hat{X}_k \cap \hat{X}_l = \emptyset.$ For an illustration see Figure~\ref{fig:Plaatje1}B.
\end{defn}
Since we are interested in finite sets, we can always define a disjoint convex hull decomposition (just take every point as a singleton, giving $L_k=|x_k|$). Such a decomposition is not unique. In practice we find decompositions with smaller $L_k$'s using the algorithm in Section \ref{sec:experiments}. The following definition concerns two sets, but can easily be extended to multiple sets by applying it pairwise.
\begin{defn}\label{def:sep}
	If $C(X_1)\cap C(X_2) = \emptyset$, $X_1$ and $X_2$ are called \textbf{linearly separable}.
    If $C(X_1)\cap X_2 = \emptyset$ or $X_1\cap C(X_2) = \emptyset$, $X_1$ and $X_2$ are called \textbf{convexly separable}.
    If all disjoint convex hull decompositions of $X_1$ and $X_2$ satisfy $\min (L_1,L_2)>1$, $X_1$ and $X_2$ are called \textbf{convexly inseparable}.
\end{defn}
We start by giving a generalization of Theorem 4 from \cite{an2015}. Instead of considering a rectified linear classifier activation function, we consider the more general class of functions that satisfy $f(x)= 0$ for $x\leq 0$ and $f(x)>0$ for $x>0$. We will call these functions \textbf{semi-positive}. Notice that they can be any function of $x>0$ as long as they remain positive. This generalization is straightforward and the proofs do not need to be adapted but are given here for easy reference.
\begin{thm} \label{thm:linsep}
	Let $X_1$ and $X_2$ be two convexly separable sets, with a finite number of points in $\mathbb{R}^n$. Say, $C(X_1)\cap X_2 = \emptyset$ and $X_2=\bigcup_{j=1}^{L_2} X_2^j$ with $L_2\in \mathbb{N}, X_2^j\subseteq X_2$ such that $C(X_1)\cap C(X_2^j)= \emptyset$ for each $j$. Let $w_j^T x+b_j$ be linear classifiers of $X_2^j$ and $X_1$ such that for all $j$
	\begin{align*}
		w_j^Tx+b_j \leq 0 & \quad\forall x\in X_1\\
        w_j^Tx+b_j > 0 & \quad\forall x\in X_2^j\,.
	\end{align*}
	Let $W=[w_1,\ldots, w_{L_2}]$, $b=[b_1,\ldots, b_{L_2}]^T$ and $Z_k = \{ f(W^Tx+b) \mid x\in X_k\}, k\in\{1,2\}$. Here $f$ is a semi-positive function that is applied component-wise. Then $Z_1$ and $Z_2$ are linearly separable. For this we need $L_2$ affine transformations.
\end{thm}
\begin{proof}
For all $x\in X_1$ we have that $w_j^Tx+b_j \leq 0$. So $Z_1 = \{ f(W^Tx+b)\mid x\in X_1\} = \{(f(w_j^Tx+b_j))_j\mid x\in X_1\} = \{0\}.$ Now, for an $x\in X_2$, there exists a $j$ such that $x\in X_2^j$. So, there exists a $j$ such that $w_j^Tx+b_j> 0$. Therefore, each $z\in Z_2$ has components greater or equal to zero and at least one component that is strictly greater than zero. This means $C(Z_1)\cap C(Z_2) = \emptyset$. We used $L_2$ transformations to create $Z_1$ and $Z_2$.
\end{proof}
\begin{figure*}
\includegraphics[width=\textwidth]{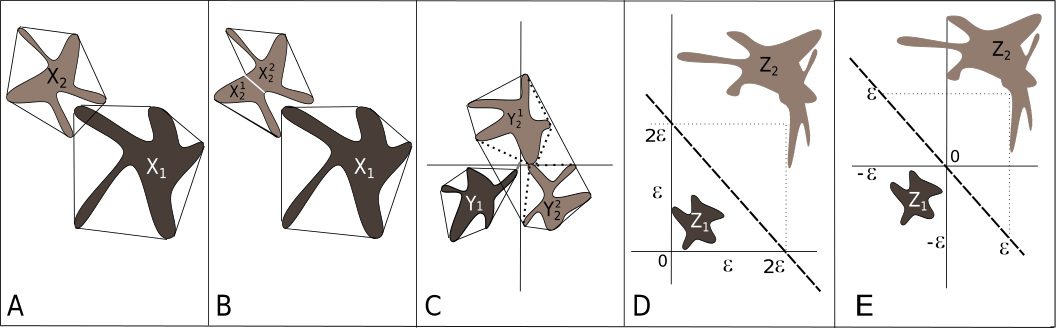}
\caption{\label{fig:Plaatje1} Illustration of the proofs of Theorems~\ref{thm:general} and~\ref{thm:leaky}. \textbf{(A)} The two original sets and their convex hulls (outline). \textbf{(B)} $X_2$ is separated in two parts such that the convex hull of each part is linearly separable from the convex hull of $X_1$. \textbf{(C)} A linear transformation sends all points in $X_1$ to the bottom-left quadrant, all points in $X_2^1$ to the upper half-plane and all points in $X_2^2$ to the right half-plane. We can determine the minimal distance between the convex hull of $Y_1$ and the convex hulls of $Y_2^1$ and $Y_2^2$. \textbf{(D)} We apply $f$. $Z_1$ will become enveloped by a regular hypercube, and $Z_2$ will lie outside a hypercube with edges that are $L_2=2$ times as long. The separating plane is drawn as a dashed line. \textbf{(E)} Equivalently, a translated picture is used in the proof of the leaky rectified linear activation function. Instead of Figure~\ref{fig:Plaatje1}D, we now have $Z_1$ below the axis.}
\end{figure*}
The initial sets that the network needs to separate are denoted by $X_k$, see Figure~\ref{fig:Plaatje1}A. After applying a linear classifier to the initial sets, these will be denoted by $Y_k$ such that after applying the transformation $w_j^T x + b_j$ on all $x\in X_1$, we get $Y_1$, see Figure~\ref{fig:Plaatje1}C. When we apply the activation function to elements in $Y_k$, we denote the resulting set by $Z_k$, shown in Figures~\ref{fig:Plaatje1}D (the constant $\epsilon$ should be taken $0$ for now). This means that a neural network with a single hidden layer with $L_2$ nodes, can transform $X_k$ into $Z_k$. The following theorem is a generalization of Theorem 5 from \cite{an2015}. Again, this is straightforward and does not require any changes to the proof. The theorem will make use of the following lemma.
\begin{lem}\label{lem:projection}
Two finite sets are linearly separable if and only if there exists a one-dimensional projection that maps the sets to linearly separable sets.
\end{lem}
\begin{proof}
Suppose we have two sets that are linearly separable. Let $l$ be the hyperplane that separates the data. Project the data on the axis that is orthogonal to the hyperplane. By this, $l$ will be collapsed into a point that lies at the threshold between the two separated sets. If we have a one-dimensional projection of the two sets, and a threshold $t$, let $m$ be the hyperplane orthogonal to the projection axis containing $t$. Then the sets will be linearly separated by $m$.
\end{proof}
\begin{thm}\label{thm:convsep}
Let $X_1$ and $X_2$ be finite and convexly inseparable.  Let $w_{ij}^T x+b_{ij}$ be linear classifiers of $X_2^j$ and $X_1^i$ such that for all $i,j$
	\begin{align*}
		w_{ij}^Tx+b_{ij} \leq 0 & \quad\forall x\in X_1^i\\
        w_{ij}^Tx+b_{ij} > 0 & \quad\forall x\in X_2^j\,.
	\end{align*}
	Let $W_i=[w_{i1},\ldots, w_{iL_2}]$ and $b_i=[b_{i1},\ldots, b_{iL_2}]$. Let $W=[W_1,\ldots, W_{L_1}]$, $b=[b_1^T,\ldots, b_{L_1}^T]^T$ and $Z_k = \{ f(W^Tx+b) \mid x\in X_k\}$ for $k\in\{1,2\}$. Also, let $Z_1^i = \{f(W^Tx+b)\mid x\in X_1^i\}$. Here $f$ is again semi-positive. Then $Z_1$ and $C(Z_2)$ are disjoint, so $Z_1$ and $Z_2$ are convexly separable. For this we need $L_1L_2$ nodes.
\end{thm}
\begin{proof}
Define $Z_{2i} = \{f(W_i^Tx+b_i) \mid x\in X_2\}$ and $Z_{1i}^t = \{f(W_i^Tx+b_i) \mid x\in X_1^t\}$. Notice that these sets are projections of $Z_2$ and $Z_1^t$. Apply Theorem \ref{thm:linsep} on $X_1^i$, $X_2$ and their images $Z_{1i}^i$ and $Z_{2i}$ under the transformation $f(W_i^T\cdot+b_i)$. Then we have 
\begin{equation*}
C(Z_{1i}^i)\cap C(Z_{2i}) = \emptyset \quad i\in\{1,\ldots, L_1\}\,.
\end{equation*} 
With Lemma \ref{lem:projection}, we then also have that
\begin{equation*} C(Z_1^i)\cap C(Z_2) = \emptyset \quad i\in\{1,\ldots, L_1\}\,.
\end{equation*} 
Since $Z_1\subset \bigcup_{i=1}^{L_1} C(Z_1^i)$, we have $Z_1\cap C(Z_2) = \emptyset$. Therefore $Z_1$ and $Z_2$ are convexly separable. We needed $L_1$ linear transformations to separate a single part of $X_2$ from all parts of $X_1$. So in total we need $L_1L_2$ transformations to create $Z_1$ and $Z_2$.
\end{proof}
From Theorem~\ref{thm:convsep} and \ref{thm:linsep} we can conclude that any two sets that are disjoint, can be made linearly separable by a network with two hidden layers that applies the function $f$ as above and has $L_2L_1$ and $L_1$ nodes per respective layer.\\

\section{A general upper bound}\label{sec:main}

We can generalize Theorems~\ref{thm:linsep} and \ref{thm:convsep} to a larger set of activation functions for which we need the following lemma. For simplicity we define $0/0=0$.
\begin{lem}\label{lem:f}
For a given $\delta>0$ and a fixed $L_2>0$, let $f:\mathbb{R}\rightarrow\mathbb{R}$ be increasing with a left asymptote to zero and $\inf_{x_0} \frac{f(x_0)}{f(x_0 +\delta)}< \frac{1}{L_2}$. Then $\exists x_0\in\mathbb{R}, \epsilon>0$ such that $\forall x\leq x_0: f(x)\in [0,\epsilon]$ and $\forall x \geq x_0+\delta : f(x)>L_2\epsilon$.
\end{lem}
\begin{proof}
Choose $x_0$ and $\epsilon$ such that $f(x_0)=\epsilon$ and $\frac{f(x_0)}{f(x_0 +\delta)}< \frac{1}{L_2}$. Let $x\leq x_0$. Then $f(x)\leq f(x_0) = \epsilon$. Let $x\geq x_0+\delta$, then $f(x)\geq f(x_0+\delta)> f(x_0)L_2 = \epsilon L_2$.
\end{proof}
Lemma \ref{lem:f} puts a constraint on the speed with which the function $f$ increases (near $-\infty$). We need $f(x_0+\delta)\geq L_2 f(x_0)$. So if we move by $\delta$, the value of the function will be multiplied by $L_2$. Notice that this is a very rapidly growing function. If a function does not satisfy this constraint, we find there is a minimum distance $\delta$ needed between the $C(Y_1)$ and $C(Y_2^1)\cup C(Y_2^2)$. Lemma \ref{lem:f} also implies that the function should have a left asymptote to zero, however, we can shift an activation function with a different left asymptote such that this holds and then shift it back later using Corollary~\ref{cor:shift}. This way, the lemma, and therefore Theorems~\ref{thm:general} and \ref{thm:general2}, holds for all commonly used activation functions. We will compute the distance $\delta$ for the sigmoid, hyperbolic tangent, rectified linear function and leaky rectified linear function in Corollary~\ref{cor:deltas}.\\
~\\
We will from now on define $\delta = \min_j \delta_j$ where 
\begin{equation}\label{eqn:delta}
\delta_j = \inf_{x,y} \{ \|x-y\| \mid x\in C(Y_1), y\in C(Y_2^j) \}
\end{equation}
is the smallest distance between the convex hulls of two sets. 
\begin{thm}\label{thm:general}
	Let $X_1$ and $X_2$ be two convexly separable sets, with a finite number of points in $\mathbb{R}^n$. So $C(X_1)\cap X_2 = \emptyset$ and $X_2=\bigcup_{j=1}^{L_2} X_2^j$ with $L_2\in \mathbb{N}, X_2^j\subseteq X_2$ such that $C(X_1)\cap C(X_2^j)= \emptyset$ for each $j$. Let $f:\mathbb{R}\rightarrow\mathbb{R}$ be increasing with a left asymptote to zero and define $\delta $ as in Equation \ref{eqn:delta}, such that \begin{equation*} \inf_{x_0} \frac{f(x_0)}{f(x_0 +\delta)}< \frac{1}{L_2}\,.
\end{equation*} For $x_0$ satisfying this inequality, let $w_j^T x+b_j$ be linear classifiers of $X_2^j$ and $X_1$ such that for all $j$
	\begin{align*}
		\sup_{x\in X_1} \{w_j^Tx+b_j\} = x_0, & ~\\
        w_j^Tx+b_j \geq x_0+\delta & \quad\forall x\in X_2^j\,.
	\end{align*}
	Let $W=[w_1,\ldots, w_{L_2}]$, $b=[b_1,\ldots, b_{L_2}]^T$ and $Z_k = \{ f(W^Tx+b) \mid x\in X_k\}, k\in\{1,2\}$.  Then $Z_1$ and $Z_2$ are linearly separable.
\end{thm}
\begin{proof}
Choose, using Lemma \ref{lem:f}, an $x_0\in \mathbb{R}$ and $\epsilon>0$ such that for all $x\leq x_0$ we have $f(x)\leq \epsilon$ and for all $x\geq x_0+\delta$ we have $f(x)>L_2\epsilon$. For all $x\in X_1$ we have $w_j^Tx+b_j\leq x_0$. So $f(w_j^Tx+b_j)\leq \epsilon$. Therefore, $Z_1$ is contained in a positive hypercube $[0,\epsilon]^{L_2}$. For all $x\in X_2$ there is a $j$ such that $x\in X_2^j$. For $x\in X_2^j$ we have that $w_j^Tx+b_j\geq x_0 + \delta$. So $f(w_j^Tx+b_j)> L_2\epsilon$ for at least one coordinate, and all other coordinates are larger than $0$. Therefore, $Z_2 \subseteq [0,\infty)^{L_2}\backslash [0,L_2\epsilon]^{L_2}$, see Figure~\ref{fig:Plaatje1}D. The convex hulls of these two sets can be separated by the hyperplane $\sum_{i=1}^{L_2}x_i = \epsilon L_2$. This is because the convex hull of $Z_1$ is contained in the hypercube with edges $\epsilon$ and the convex hull of $Z_2$ is bounded by the separating hyperplane.
\end{proof}
\begin{thm}\label{thm:general2}
Let $X_1$ and $X_2$ be finite and convexly inseparable. Let $f:\mathbb{R}\rightarrow\mathbb{R}$ be increasing with a left asymptote to zero, define $\delta $ as in Equation~\ref{eqn:delta} such that $\inf_{x_0} \frac{f(x_0)}{f(x_0 +\delta)}< \frac{1}{L_2}$. For $x_0$ satisfying this inequality, let $w_{ij}^T x+b_{ij}$ be linear classifiers of $X_2^j$ and $X_1^i$ such that for all $i,j$
	\begin{align*}
		\sup_{x\in X_1^i} \{w_{ij}^Tx+b_{ij}\} = x_0, & ~\\
        w_{ij}^Tx+b_{ij} \geq x_0 +\delta & \quad\forall x\in X_2^j\,.
	\end{align*}
	Let $W_i=[w_{i1},\ldots, w_{iL_2}]$ and $b_i=[b_{i1},\ldots, b_{iL_2}]$. Let $W=[W_1,\ldots, W_{L_1}]$, $b=[b_1^T,\ldots, b_{L_1}^T]^T$ and $Z_k = \{ f(W^Tx+b) \mid x\in X_k\}$ for $k\in\{1,2\}$. Also, let $Z_1^i = \{f(W^Tx+b)\mid x\in X_1^i\}$. Then $Z_1$ and $Z_2$ are convexly separable.
\end{thm}
\begin{proof}
Define $Z_{2i} = \{f(W_i^Tx+b_i) \mid x\in X_2\}$ and $Z_{1i}^t = \{f(W_i^Tx+b_i) \mid x\in X_1^t\}$.
Notice that these sets are projections of $Z_2$ and $Z_1^t$. Apply Theorem~\ref{thm:general} on $X_1^t$, $X_2$ and their images $Z_{1i}^i$ and $Z_{2i}$ under the transformation $f$. Then we have \begin{equation*} C(Z_{1i}^i)\cap C(Z_{2i}) = \emptyset \quad i\in\{1,\ldots, L_1\}\,.
\end{equation*} With Lemma~\ref{lem:projection}, we then also have that
\begin{equation*} C(Z_1^i)\cap C(Z_2) = \emptyset \quad i\in\{1,\ldots, L_1\}\,.
\end{equation*} Since $Z_1\subset \bigcup_{i=1}^{L_1} C(Z_1^i)$, we have $Z_1\cap C(Z_2) = \emptyset$. Therefore $Z_1$ and $Z_2$ are convexly separable.
\end{proof}
So we see that both Theorems~\ref{thm:linsep} and \ref{thm:convsep} can be generalized to increasing functions with a left asymptote to zero. We still need two layers with $L_1L_2$ and $L_1$ nodes respectively. However, we also need a minimal separation $\delta$ between the convex hulls of the two sets (in Euclidean distance) after applying the first linear transform. We formalize this in the following theorem:
\begin{thm}\label{thm:universe}
Given disjoint finite sets $X_1$ and $X_2$ with a disjoint convex hull decomposition with $L_1$ and $L_2$ sets in the partitions, and given an increasing activation function $f$ with a left asymptote to zero, we can linearly separate $X_1$ and $X_2$ using an artificial neural network with an input layer, a layer with $L_1L_2$ hidden nodes, a layer with $L_1$ hidden nodes and an output layer.
\end{thm}
\begin{proof}
We can assume $X_1$ and $X_2$ are convexly inseparable and have linear classifiers $w^T_{ij}x+b_{ij}$ as in Theorem \ref{thm:convsep}. The corresponding $\delta$ is always greater than zero, and scales with $(w,b)$. Since $f\neq0$ we can scale $\delta$ such that $\inf_{x_0}\frac{f(x_0)}{f(x_0+\delta)}<\frac{1}{L_2}$. Then apply Theorem \ref{thm:general2}. We need $L_1L_2$ affine transformations for separating the $L_2$ parts of $X_2$ and the $L_1$ parts of $X_1$. Then $f$ is applied to all transformations. A neural network can do this by learning the weights and biases of the affine transformations and then applying $f$. Now we have $L_2$ pairs of convexly separable sets, which can be linearly separated using Theorem \ref{thm:general}. For each $j$ we need to find an affine plane that separates $X_2^j$ from $X_1$. This means we have to learn $L_2$ affine transformations before applying $f$, which can be done by a neural network with $L_2$ nodes. Now we have two linearly separable sets, which can be separated by using a linear classifier as the output layer, which proves the theorem. 
\end{proof}
Note that this proof implies that we can separate $X_1$ and $X_2$ independent of the distance between $Y_1$ and $Y_2$, so independent of $\delta$. The learning algorithm should be able to scale the weights and biases such that the sets can be separated no matter how small $\delta$ was originally.\\
~\\
We can also prove a similar theorem for the leaky rectified linear activation function, which does not have a left asymptote to zero. However, we need to prove Lemma \ref{lem:leakylemma} first. The diameter of a set $A$ is defined as $\textnormal{diam}(A) = \sup \{\|x-y\| \mid x,y\in A\}$.
 
\begin{lem}\label{lem:leakylemma}
Suppose $D=\textnormal{diam}(Y_1\cup Y_2)$, $\delta$ as in Equation~\ref{eqn:delta}, and $f(x) = c_2 x$ for $x\geq 0$ and $f(x) = c_1 x$ for $x\leq 0$, where $c_2>c_1$. Then $\mu(\delta, D) \triangleq \inf_{x_0} \frac{f(x_0)-f(x_0-D)}{f(x_0+\delta)-f(x_0-D)}$ is reached at $x_0 = 0$.
\end{lem}
\begin{proof}
We have four cases:
\begin{itemize}
\item[(a)] $x_0<0, x_0+\delta<0$, then $x_0-D<0$:
\begin{eqnarray*}
\mu(\delta, D) 
&=& \inf_{x_0}\frac{c_1x_0-c_1x_0+c_1D}{c_1x_0+c_1\delta-c_1x_0+c_1D} \\
&=& \inf_{x_0}\frac{1}{\frac{\delta}{D}+1} = \frac{1}{\frac{\delta}{D}+1}
\end{eqnarray*}
\item[(b)] $x_0<0, x_0+\delta\geq0$, then $x_0-D<0$:
\begin{eqnarray*}
\mu(\delta, D)  
&=& \inf_{x_0}\frac{(c_1-c_1)x_0+c_1D}{(c_2-c_1)x_0+c_2\delta+c_1D} \\
&=& \inf_{x_0}\frac{1}{\frac{(c_2-c_1)x_0}{c_1D}+\frac{c_2\delta}{c_1D}+1} \\
&=& \frac{1}{\frac{c_2\delta}{c_1D}+1}
\end{eqnarray*}
for $x_0$ increasing to zero.
\item[(c)] $x_0\geq 0, x_0-D<0$, then $x_0+\delta\geq 0$:
\begin{eqnarray*}
\mu(\delta, D) 
&=& \inf_{x_0}\frac{c_1x_0-c_1x_0+c_1D}{(c_2-c_1)x_0+c_2\delta+c_1D} \\
&=&  \inf_{x_0}\frac{\frac{(c_2-c_1)x_0}{c_1D}+1}{\frac{(c_2-c_1)x_0}{c_1D}+\frac{c_2\delta}{c_1D}+1}
\end{eqnarray*}
which is an increasing function on the interval $[0,D)$. Therefore the infimum will be at $x_0=0$.
\item[(d)] $x_0\geq0, x_0-D\geq 0$, then $x_0+\delta>0$:
\begin{eqnarray*}
\mu(\delta, D) 
&=& \inf_{x_0}\frac{c_2x_0-c_2x_0+c_2D}{c_2x_0+c_2\delta-c_2x_0+c_2D} \\
&=&  \inf_{x_0}\frac{1}{\frac{\delta}{D}+1} = \frac{1}{\frac{\delta}{D}+1}
\end{eqnarray*}
Since cases (a) and (d) are equal, and since $c_2/c_1>1$ we see that the infimum is assumed at the value $x_0=0$.
\end{itemize}
\end{proof}
Now we are ready to prove the following theorem for leaky rectified linear functions. Because of Lemma \ref{lem:leakylemma} we can assume $x_0=0$.
\begin{thm}\label{thm:leaky}
Suppose we have $X_1$ and $X_2$ as in Theorem \ref{thm:linsep}. Define $\delta$ as in Equation~\ref{eqn:delta}. Let $f(x)=c_1x$ for $x\leq 0$ and $f(x) = c_2x$ for $x\geq 0$ be increasing with \begin{equation*} \frac{-f(-D)}{f(\delta)-f(-D)}< \frac{1}{L_2}.
\end{equation*} Then $Z_1$ and $Z_2$ as defined in Theorem~\ref{thm:linsep} are linearly separable.
\end{thm}
\begin{proof}
Let $\epsilon = -f(-D)$. Then $\forall -D\leq x\leq 0$ we have $-\epsilon \leq f(x)\leq 0$. And $\forall x>\delta$ we have $f(x)>(L_2-1)\epsilon$. For all $x\in X_1$ we know $-D\leq w_j^Tx+b_j\leq 0$. Therefore for all $x\in Y_1$ we have $-\epsilon \leq f(x)\leq 0$ and for all $x\in Y_2$ we have that $f(x)\geq -\epsilon$ and there exists a $j$ such that $f(x_j)> (L_2-1)\epsilon$. See Figure~\ref{fig:Plaatje1}E. Therefore the convex hulls of $Z_1$ and $Z_2$ are disjoint.
\end{proof}
We will not prove a version of Theorem \ref{thm:convsep} for the leaky rectified linear activation function because this is straightforward and the proof is the same as the proof of Theorem \ref{thm:convsep}. But we can conclude that for leaky rectified linear activation functions a network that consists of two layers and $L_2$ and $L_1L_2$ nodes respectively, can achieve linear separability. If $\delta$ is not large enough, there are two options now: the network could learn to scale the weights and biases appropriately, or function could be adjusted manually by increasing the fraction ${c_2}/{c_1}$. In the next section we will explore some consequences of these results. We will also provide a way to calculate $L_1$ and $L_2$.

\section{Corollaries and a practical algorithm}\label{sec:corr}

We can generalize the results from Section~\ref{sec:main} to any number of sets (Corollary~\ref{cor:multiple}), by using the similar result in Sections~3.4 and 3.5 from \cite{an2015} as a foundation. After stating this result we will show that we can apply any translation to the function $f$ in the above theorems while retaining their validity (Corollary~\ref{cor:shift}). Then we will provide a cheap way to estimate $L_1$ and $L_2$ in the disjoint convex hull decomposition (Algorithm~\ref{alg:algorithm1}) and we will calculate $\delta$, for the most commonly used activation functions (Corollary~\ref{cor:deltas}).
\begin{cor}\label{cor:multiple}
For any number of sets, the above holds, with adjusted $L_1$ and $L_2$. By using the result from Section 3.4 and 3.5 on multiple sets from~\cite{an2015} as a foundation, it is easy to see that the same reasoning will apply to Theorems \ref{thm:linsep}, \ref{thm:convsep}, \ref{thm:general}, \ref{thm:general2} and \ref{thm:leaky}.
\end{cor}
We can generalize the theorems still a little more by showing that they also hold for translated versions of the activation function that satisfies the constraints.
\begin{cor}\label{cor:shift}
Suppose $f(-\infty)=c$, $f$ is increasing and we have $\inf_{x_0}\frac{f(x_0)-c}{f(x_0+\delta)-c}\leq\frac{1}{L_2}$. Then Theorem \ref{thm:general} still holds. Moreover, for left and right translations of $f$, Theorem \ref{thm:general} still holds.
\end{cor}
\begin{proof}
Define $g=f-c$. Then $g(-\infty)=0$ and $\inf_{x_0}\frac{g(x_0)}{g(x_0+\delta)}\leq \frac{1}{L_2}$. Therefore the theorem holds for $g$. Adding a constant to the linear separable sets $Z_1$ and $Z_2$ does not affect their separability. So the theorem holds for $f$.\\
Let $g(x) = f(x+c)$ be a translated version of $f$. If we apply $g(y)= g(w_j^Tx+b_j)=f(w_j^Tx+b_j+c)$ we see that we could just subtract $c$ from $x_0$ to get back to the original theorem. Therefore left and right translation of functions is allowed. 
\end{proof}
We need a way to estimate $L_1$ and $L_2$ for arbitrary datasets. Since it is difficult to decompose the sets in a high dimensional space, we found a way to do it in a low dimensional space. This will allow for a rough upper bound on $L_1$ and $L _2$ but does not guarantee that the smallest disjoint convex hull decomposition can be found.
\begin{lem}\label{lem:projection2}
If we have a disjoint convex hull decomposition of a projection of our dataset, this partition will also form a disjoint convex hull decomposition of the original dataset.
\end{lem}
\begin{proof}
Suppose $P(X_1)$ and $P(X_2)$ are $n$-dimensional projections of $X_1$ and $X_2$. Assume we have disjoint convex hull decompositions $\widehat{P(X_1)} = \bigcup_{j=1}^{L_1} C(P(X_1)^j)$ and $\widehat{P(X_2)} = \bigcup_{j=1}^{L_2} C(P(X_2)^j)$ such that $\widehat{P(X_1)}\cap \widehat{P(X_2)}=\emptyset$. Now take $\widehat{X_1} = \bigcup_{j=1}^{L_1} C(X_1^j)$ such that $P(X_1)^j=P(X_1^j)$ and $\widehat{X_2} = \bigcup_{j=1}^{L_2} C(X_2^j)$ such that $P(X_2)^j=P(X_2^j)$. Then we see that $C(P(X_1^j))\cap C(P(X_2^i))=\emptyset$. Therefore $P(C(X_1^j))\cap P(C(X_2^i))=\emptyset$. So we can conclude that $C(X_1^j)\cap C(X_2^i)=\emptyset$. So also $\widehat{X_1}=\bigcup_{j=1}^{L_1} C(X_1)^j$ and $\widehat{X_2}=\bigcup_{j=1}^{L_2} C(X_2)^j$ are a disjoint convex hull decomposition.
\end{proof}
We can estimate the number of sets in the convex hull decomposition using Lemma~\ref{lem:projection2} as follows: Take a random projection of the datasets, preferably a one-dimensional projection. Then find the disjoint convex hull decomposition of this projection alone. This is easy in one dimension as it can be done by counting how often one switches from one set to the other when traversing through the projection. This number is an upper bound for $L_1$ and $L_2$, but a very coarse one, as we are using a random projection. So it is necessary to repeat this for many more random projections and minimize for $L_1$ and $L_2$. We use this procedure in Algorithm~\ref{alg:algorithm1} to find a reasonable estimate for $L_1$ and $L_2$. When taking the difference of the means of the two sets instead of the random projection, we may find a large portion of either sets at the extreme ends, as in Figure~\ref{fig:sets}. We utilize this in our algorithm. We can prove that this algorithm will actually give a disjoint convex hull decomposition.
\begin{algorithm}
\caption{\label{alg:algorithm1} Estimating L1 and L2}
\begin{algorithmic}[1]
\State $\text{\textbf{Input}: two sets, X and Y, of n-dimensional data,}$\\ 
$\quad\quad\quad~~\text{the sets are finite and disjoint.}$
\State $\text{\textbf{while} size of overlap is smaller than size of previous overlap}$
\State $\quad\text{\textbf{count} the number of times we do this}$
\State $\quad\text{\textbf{calculate} mx $\leftarrow$ mean of X}$\\
$\quad\quad\quad\quad\quad~~\text{my $\leftarrow$ mean of Y}$
\State $\quad\text{\textbf{for} x in X, y in Y}$
\State $\quad\quad\text{\textbf{project} x and y on my-mx}$
\State $\quad\text{\textbf{calculate} cx $\leftarrow$ maximum of the projection of X}$\\ 
$\quad\quad\quad\quad\quad~~\text{cy $\leftarrow$ minimum of the projection of Y}$
\State $\quad\text{\textbf{create} overlapX $\leftarrow$ all x with projection at least cy}$\\
$\quad\quad\quad\quad~\text{overlapY $\leftarrow$ all y with projection at most cx}$
\State $\quad\text{\textbf{replace} X by overlapX, Y by overlapY}$
\State $\text{\textbf{for} t in \text{number of random projections}}$
\State $\quad\text{\textbf{create} random vector}$
\State $\quad\text{\textbf{for} x,y in overlapX and overlapY}$
\State $\quad\quad\text{\textbf{calculate} px $\leftarrow$ projection of x on random vector}$\\
$\quad\quad\quad\quad\quad\quad~~\text{py $\leftarrow$ projection of y on random vector}$
\State $\quad\text{\textbf{count} the number of set changes from px to py}$
\State $\quad\text{L1 $\leftarrow$ number of set changes+count}$
\State $\quad\text{L2 $\leftarrow$ number of set changes+1+count}$
\State $\text{\textbf{minimize} L1 and L2}$
\end{algorithmic}
\end{algorithm}

\begin{thm}
Algorithm~\ref{alg:algorithm1} gives a disjoint convex hull decomposition of the input sets and has complexity of order $\mathcal{O}(n^2)$ where $n$ is the size of the input sets.
\end{thm}
\begin{proof}
Define $X_1$ and $Y_1$ as the parts of $X$ and $Y$ that are outside $cx$ and $cy$. Call the overlap $Z_1$. Notice that $X_1, Y_1$ and $Z_1$ have disjoint convex hulls because their projections have disjoint convex hulls, with Lemma \ref{lem:projection}. Within the overlap $Z_1$ we can again compute new $cx$ and $cy$ and we call the parts of $X$ and $Y$ that are in $Z_1$ and outside $cx$ and $cy$, $X_2$ and $Y_2$. We call the overlap $Z_2$. Continue in this way to obtain $X_j,Y_j,Z_j$ for $j\in{1,...n}$. Then $C(Z_j)\cap C(X_{j})=\emptyset$ and $C(Z_j)\cap C(Y_{j})=\emptyset$. For all $j> k$ we have that $C(X_j)\subseteq C(Z_k)$ and $C(Y_j)\subseteq C(Z_k)$. Therefore, $C(X_j)\cap C(X_{k})=\emptyset$ and $C(X_j)\cap C(Y_{k})=\emptyset$. Equivalently for $Y_j$, $C(Y_j)\cap C(Y_{k})=\emptyset$ and $C(Y_j)\cap C(X_{k})=\emptyset$. So then $X_1, \ldots, X_n, Y_1, \ldots, Y_n, Z_n$ have disjoint convex hulls. When the while-loop terminates, we may still have a non-empty overlap. So suppose $Z_{n}\neq \emptyset$. Then using a random projection, we find a disjoint convex hull decomposition of $Z_{n}$. The convex hulls of these sets will all be contained within the convex hull of $Z_{n}$ and therefore disjoint from the convex hulls of all previously found sets. Therefore, the algorithm gives a disjoint convex hull decomposition of the input sets.\\
~\\
The algorithm has complexity $\mathcal{O}(n^2)$. The worst-case scenario for the while-loop contributes a factor $n$ and computing the inner-products also contributes a factor $n$. It is fair to mention the algorithm also depends on the dimension of the data and on the number of random projections that is used. Both can be quite large. The number of random projections needs to be significantly larger than the dimension to get good results. Also notice that adding \textit{count} to $L_1$ and $L_2$ in lines 20 and 21 is na\"ive and can easily be improved.
\end{proof}
\begin{figure}
\centering
\includegraphics[width=0.6\textwidth]
{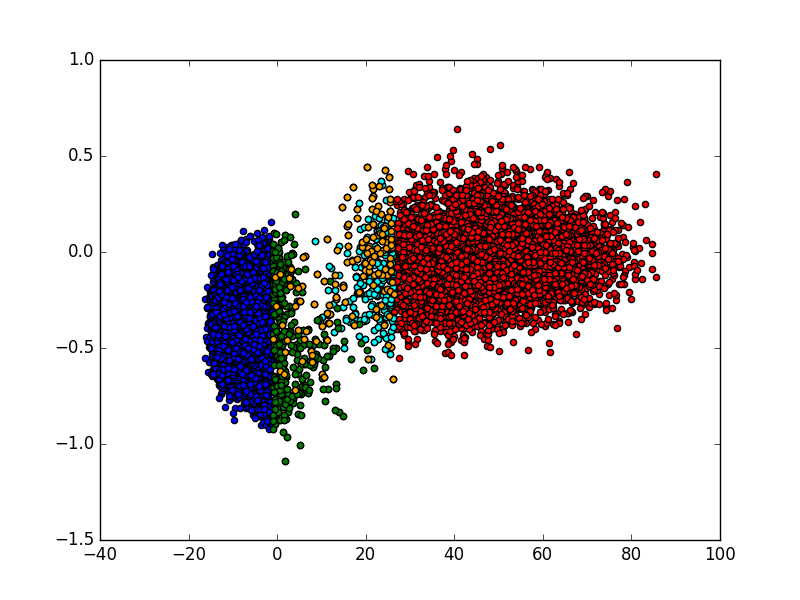}
\caption{\label{fig:sets} Here we see a projection of the first two number classes of the MNIST dataset. The ones are in red, the zeros in blue. The x-axis is the axis between the means of the two sets. Applying the algorithm gives an overlap (cyan and green). Again calculating the means and the resulting overlap gives the orange set. For illustrative reasons we did not separately indicate the ones and the zeros in the orange set.}
\end{figure}
\begin{cor}\label{cor:deltas}
For a dataset with convex hull decomposition in $L_2$ sets, recall that the minimal distance between the $C(Y_1)$ and $C(Y_2^j)$ is $\delta=\inf_j \delta_j$, see Equation~\ref{eqn:delta}. For the sigmoid the minimal $\delta$ needed for separation equals $\ln(L_2)$. For a shifted hyperbolic tangent the minimal $\delta$ equals $\frac{1}{2}\ln(L_2)$. For the ReLU the minimal $\delta$ equals $0$. For the leaky rectified linear activation function $\delta = (L_2+1)\frac{c_2}{c_1}D$. Or equivalently, $\frac{c_2}{c_1}=\frac{\delta}{D(L_2-1)}$, then we are able to separate any two sets with a leaky rectified linear activation function.
\end{cor}
\begin{proof}
Note that proving that the limit becomes smaller than $1/L_2$ implies that the infimum also becomes smaller than $1/L_2$. For practical purposes we will use the limit in this proof.\\

\textbf{Sigmoid.} The sigmoid function is written as $\sigma(x)=\frac{e^x}{1+e^x}$. If we calculate \begin{equation}\label{eqn:sigm}
\frac{\sigma(x_0)}{\sigma(x_0+\delta)} = \frac{e^{x_0}}{1+e^{x_0}}\frac{1+e^{x_0}e^\delta}{e^{x_0}e^\delta} = \frac{e^{-\delta}+e^{x_0}}{1+e^{x_0}}
\end{equation}
and then take the limit $x_0\rightarrow -\infty$ we see that equation \ref{eqn:sigm} goes to $e^{-\delta}$. To get this smaller than $1/L_2$ we need $\delta>\ln(L_2)$.\\

\textbf{Hyperbolic tangent.} We start by writing a shifted hyperbolic tangent $\tanh(x)+1$ out in terms of exponentials. If we calculate \begin{equation} \label{eqn:tanh}
	\frac{\tanh(x_0)+1}{\tanh(x_0+\delta)+1} = \frac{2}{1+e^{-2x_0}}\frac{1+e^{-2x_0}e^{-2\delta}}{2}
	= \frac{1+e^{-2x_0}e^{-2\delta}}{1+e^{-2x_0}}
\end{equation}
and then take the limit $x_0\rightarrow -\infty$ we get that equation \ref{eqn:tanh} goes to $e^{-2\delta}$. To get this smaller than $1/L_2$ we need $\delta>\frac{1}{2}\ln(L_2)$. So also for the hyperbolic tangent we have with Corollary \ref{cor:shift} that $\delta>\frac{1}{2}\ln(L_2)$.\\

\textbf{Rectified linear function.} We did not need any $\delta$ in the proof for the rectified linear function, so the minimal $\delta$ equals zero.\\

\textbf{Leaky rectified linear activation function.} 
With Lemma \ref{lem:leakylemma} we get: \begin{equation}
\frac{-f(-D)}{f(\delta)-f(-D)}= \frac{Dc_1}{\delta c_2+Dc_1},
\end{equation} where $f$ denotes the leaky rectified linear function. To get this smaller than $\frac{1}{L_2}$ we need $\delta = (L_2+1)\frac{c_2}{c_1}D$.
\end{proof}

\section{Experimental validation}\label{sec:experiments}
We could validate the theory by showing that in fact a network of the estimated size can perfectly learn to classify the two training sets. For this we will need a proper estimate of $L_1$ and $L_2$, but it is difficult to get a tight approximation. We also need a perfect training framework, which of course does not exist. So working with the tools we have, we show an estimate for $L_1$ and $L_2$ provided by the algorithm. It is a good estimate, but can definitely be improved. We train the network using stochastic gradient descent for a long number of epochs. The loss does converge to a number close to zero, but does not become zero. This we believe is caused by imperfect training.\\
~\\
We tested the ideas in Sections~\ref{sec:main} and \ref{sec:corr} empirically. We trained several networks with different sizes and activation functions on the first two classes (number classes 0 and 1) of the MNIST dataset~(\cite{MNIST}). We calculated the minimal distance between these two sets and found $\delta = 3.96$. This is a sufficient distance for any of the activation functions we used, which means the network is able to use weights close to one. Next we estimated $L_1$ and $L_2$. For this dataset with more than 12000 data points, we found $L_1=6$ and $L_2=6$. That would mean that a network with $36$ nodes in the first and $6$ nodes in the second layer would be sufficient to linearly separate the data in the two sets.\\
~\\
Several hidden layer sizes were tested. All networks had a depth of three, hence four layers of nodes -- an input layer with 784 nodes, two hidden layers with the sizes mentioned before, and an output layer with two nodes which acts as a classifier. Linear separability, as discussed in this paper, precisely means that this output layer can classify the input sets perfectly.\\
~\\
We compared the ReLU, sigmoid, leaky ReLU and tanh networks trained for 150 epochs using stochastic gradient descent optimization. For the leaky ReLU the slope was set to the standard value of 0.2. We implemented the linear classifier multi-layer perceptron in the neural network framework Chainer v2.0 (\cite{Chainer}). We regard the training capabilities of this framework as a black box sufficient for our simulation needs. The results are displayed in Figure~\ref{fig:main}.\\
\begin{figure}[t]
\centering
\includegraphics[trim= 10 110 30 20, clip=true, width=0.6\textwidth]{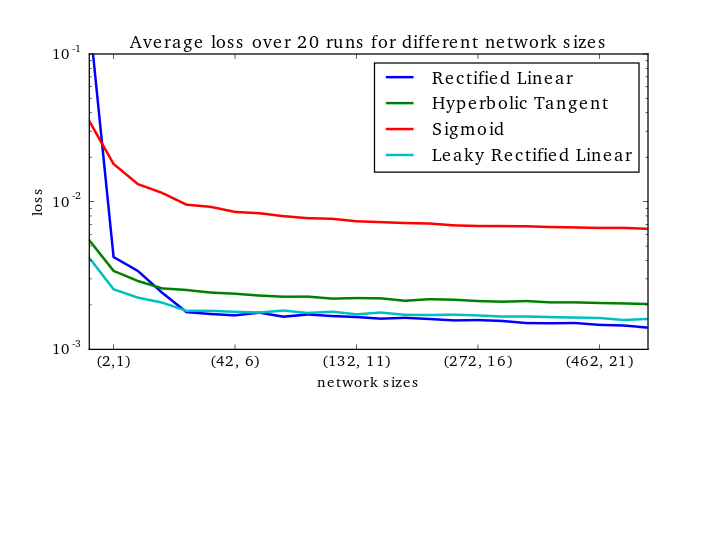}
\caption{\label{fig:main}Average loss over 20 runs on a logarithmic scale. The activation functions we used are the ReLU, the tanh, the sigmoid and the leaky ReLU. We trained 25 networks of different sizes using the above functions as activation functions. The numbers in brackets on the x-axis denote the sizes of the hidden layers. The size of the input layer was 784, the size of the output layer was two. The network was trained with 150 epochs and a batch size of 150.}
\end{figure}
~\\
Indeed as expected, the network with the hidden layer sizes estimated based on the proposed theoretical analysis (i.e. $(36,6)$) performs very well. We see clearly that the losses barely decrease for larger networks. The error is not yet zero for the predicted network but this may be explained by the imperfect training. The ReLU network performs poorly for the smallest network. This may be explained by the fact that the ReLU maps a lot of information to zero even though it has the smallest $\delta$ of the tested activation functions. The sigmoid consistently has a larger loss than the other functions. This is not necessarily predicted by the theory since the sigmoid's $\delta$ is only a factor $2$ larger than the hyperbolic tangent's $\delta$. Also interesting is the very good performance of the leaky ReLU network. This could be caused by not mapping a lot of information to zero like the ReLU as well as having two options to compensate for the $\delta$.\\
~\\
All activation functions seem to imply that there exists a slightly smaller network that can achieve linear separability on the test set. A better algorithm for determining $L_1$ and $L_2$ can probably confirm this.\\

\section{Conclusion}
\label{sec:disc}
In conclusion we can say that in theory we are now able to find a network with two hidden layers that will perfectly solve any finite problem. In practice we see that the training error does not decrease to zero. We believe that this is caused by imperfect training.\\
~\\
The practical contribution of this article is heuristic. It is widely believed that deep neural networks need less nodes in total than shallow neural networks to solve the same problem. Our theory presents an upper bound on the number of nodes that a shallow neural network will need to solve a certain classification problem. Therefore, a deep neural network will not need more nodes. The theory does not give an optimal architecture, nor a minimum on the number of nodes. Still it is useful to have an inkling about the correct network size for solving a certain problem.\\
~\\
Contrary to what~\cite{an2015} claim, their theory does not show why ReLU networks have a superior performance. We extended their theory to all commonly used activation functions. Only the leaky rectified linear networks seem to be at a disadvantage, but test results show the opposite. We think the differences between the functions may be caused by the scaling that needs to be done during learning. The linear functions and also the hyperbolic tangent are very easy to scale. Tweaking the sigmoid to the best slope can be quite difficult.\\
~\\
Some issues which we have not addressed in this article are worth mentioning. For example, we cannot make any statements about generalization performance of the networks. Of course, it is generally known that a network with too many parameters will not generalize well. So it is wise to use a network that is as small as possible, or even a bit smaller. This paper contributes an estimate for the number of nodes that is an absolute maximum. It should never be necessary to use more nodes than this estimate. We  do not give a necessary number of nodes but rather an upper bound. A bound that is necessary and sufficient would be optimal, but this is a much harder problem to solve.\\
~\\
Another problem is that we do not know what will happen if we use too few nodes. The number of nodes that we estimated will guarantee linear separability. If the number of nodes is too small to achieve linear separability, performance on the training set will be reduced, but  it is difficult to say anything about performance on the test set. We also do not know what will happen to the number and distribution of nodes as we increase the number of layers. An extension of the theory to an arbitrary number of layers would be very interesting.\\
~\\
Furthermore, in the simulations we cannot guarantee that the learning algorithm achieves zero error, even though it is possible in theory. The reason is that the algorithm does not always find the absolute minimum. Therefore it is hard to judge from the results whether the predicted network size is performing as expected.\\
~\\
Even though we already find small $L_1$ and $L_2$, more elaborate simulations could use another algorithm to find the convex hull decomposition. Random projections are cheap to use, but they will always find a pair $L_1$, $L_2$ such that $L_2=L_1+1$. (We found $L_1=L_2$ since no random projections were necessary.) This is a serious constraint because the first layer of the network consists of $L_1L_2$ nodes, and will therefore always be very large if $L_1$ and $L_2$ are similar size. An idea would be to use a method that uses higher dimensional projections. It is also not guaranteed that Algorithm \ref{alg:algorithm1} performs well on other input sets. A better algorithm might perform well on all types of input sets. \\
~\\
The results show a stunning performance of the leaky ReLU activation. More research is needed to understand why this is the case. There clearly is more to the performance of a neural network than revealed in this article. Still, it is an important result to have an estimation of sufficient network sizes for certain activation functions. It would also be interesting to see the effect of the slope of a leaky ReLU and the distance between the datasets on the performance of the network.\\
~\\
This paper provides a heuristic explanation why ReLU and perhaps leaky ReLU networks are easier to train than tanh and sigmoid networks. We give an upper bound on the number of nodes that is needed to achieve linear separability on the training set for feedforward networks with two hidden layers. It is still unclear how this generalises to more layers, which poses an interesting question for further research. Furthermore, our theory does not yet address convolutional networks, however it does represent a foundation for exploring their superior performance in an extension of this work.

\bibliographystyle{apalike}
\bibliography{sample}

\end{document}